\documentclass[twocolumn,10pt]{asmeArxiv}

\usepackage{graphicx} 
\usepackage{hyperref}
\usepackage{mathtools}
\usepackage{cmll}
\usepackage{cite}
\usepackage{nicefrac}
\usepackage{amsmath}
\usepackage{amssymb}
\usepackage{amstext}
\usepackage{color}
\usepackage{url}
\renewcommand{\vector}[1]{\mbox{\boldmath$#1$}}
\renewcommand{\vec}[1]{\mathbf{#1}}
\usepackage{multirow}
\newcommand{\btau}   {\mbox{\boldmath $\tau$}}

\newcommand{\vom} {\boldsymbol{\omega}}

\newcommand{\bI}{\mbox{\boldmath $I$}}
\newcommand{\bR}{\mbox{\boldmath $R$}}
\newcommand{\bS}{\mbox{\boldmath $S$}}

\newcommand{\bc}{\mbox{\boldmath $c$}}
\newcommand{\bp}{\mbox{\boldmath $p$}}

\newcommand{\bzero}{\mbox{\boldmath $0$}}

\newcommand{\vg}{\vec{g}}

\newcommand{\vI}{\vec{I}}

\newcommand{\BM}{\begin{bmatrix}}
\newcommand{\EM}{\end{bmatrix}}
\newcommand{\T}{^{\!\top}}

\newcommand{\vPhi}{\boldsymbol{\Phi}}
\newcommand{\vPhid}{\dot{\vPhi}}

\newcommand{\bone}{{\bf 1}}

\newcommand{\vA}{\vec{A}}
\newcommand{\va}{\vec{a}}
\newcommand{\vB}{\vec{B}}

\newcommand{\vC}{\vec{C}}

\newcommand{\vD}{\vec{D}}
\newcommand{\dd}[2]{\frac{{\rm d}#1}{{\rm d}#2}}

\newcommand{\ddt}{\dd{}{t}}

\newcommand{\vF}{\vec{F}}

\newcommand{\vf}{\vector{f}}
\newcommand{\Bf}{{\bf f}}
\newcommand{\vH}{\vec{M}}

\newcommand{\vJ}{\vec{J}}

\newcommand{\vn}{\boldsymbol{n}}

\newcommand{\vq}{\vec{q}}
\newcommand{\vqd}{\dot{\vq}}

\newcommand{\vqdd}{\ddot{\vq}}

\newcommand{\qd}{\dot{q}}

\newcommand{\vv}{\vector{v}}
\newcommand{\bv}{{\bf v}}
\newcommand{\ba}{{\bf a}}
\newcommand{\vX}{\vec{X}}

\newcommand{\XM}[2]{{}^{#1}\vX_{#2}}
\newcommand{\XR}[2]{{}^{#1}\bR_{#2}}
\newcommand{\vx}{\vec{x}}
\newcommand{\vy}{\vec{y}}
\newcommand{\vz}{\vec{z}}
\newcommand{\beq}{\begin{equation}}
\newcommand{\eeq}{\end{equation} }

\newcommand{\rel}{\alpha}

\newcommand{\vPhiDot}{\dot{\vPhi}}

\newcommand{\BF}{\mathbf F}
\newcommand{\Mot}{\mathcal{M}}
\newcommand{\For}{\mathcal{F}}
\newcommand{\M}{\mathcal{M}}
\newcommand{\F}{\mathcal{F}}
\newcommand{\I}{\mathcal{I}}

\newcommand{\crff}{\,\overline{\!\times\!}{}^{\,*}}
\newcommand{\crf}{\times^*}
\newcommand{\crm}{\times}
\newcommand{\before}[1]{\,\preceq\, #1}
\newcommand{\after}[1]{\,\succeq\, #1}
\renewcommand{\rel}[1]{\,\sim\, #1}
\newcommand{\sumbefore}[2]{#1 \before{#2}}
\newcommand{\sumafter}[2]{#1\after{#2}}
\newcommand{\sumrel}[2]{#1\rel{#2}}
\newcommand{\f}{\Bf}
\renewcommand{\v}{\bv}
\newcommand{\w}{\mathbf{w}}
\newcommand{\Rn}{\mathbb{R}^n}
\newcommand{\Rnn}{\mathbb{R}^{n\times n}}
\newcommand{\R}{\mathbb{R}}
\newcommand{\C}{\vC}
\renewcommand{\H}{\vH}
\newcommand{\g}{\mathbf{g}}
\newcommand{\q}{\mathbf{q}}
\renewcommand{\qd}{\dot{\q}}
\newcommand{\Hd}{\dot{\H}}
\newcommand{\V}{\mathcal{V}}
\newcommand{\W}{\mathcal{W}}
\newcommand{\Vstar}{\V^*}
\newcommand{\Wstar}{\W^*}
\newcommand{\A}{\mathbf{A}}
\renewcommand{\w}{\mathbf{w}}
\renewcommand{\x}{\vx}

\newcommand{\z}{\vz}
\usepackage{bbding}
\usepackage{stackengine}
\usepackage{scalerel}
\usepackage{xcolor}
\usepackage{graphicx}
\newcommand\openbigstar[1][0.7]{%
  \scalebox{.75}{\scalerel*{%
    \stackinset{c}{-.125pt}{c}{}{\scalebox{#1}{\color{white}{$\bigstar$}}}{%
      $\bigstar$}%
  }{\bigstar}}
}
\newcommand{\changeCoord}[1]{\overline{#1}}
\newcommand{\Chris}[1]{#1^{\openbigstar[.5]}}
\newcommand{\vCc}{\Chris{\vC}}

\newcommand{\vqh}{\changeCoord{\vq}}

\newcommand{\Cchat}{\changeCoord{C}{}\Chris{}}
\newcommand{\vChat}{\changeCoord{\vC}}

\newcommand{\pq}[2]{ \frac{\partial q_{#1}}{\partial \changeCoord{q}_{#2}}}

\newcommand{\vCchat}{\changeCoord{\vC}{}\Chris{}}
\newcommand{\Hhat}{\changeCoord{\H}}
\newcommand{\Gammahat}{\changeCoord{\Gamma}}
\newcommand{\Cc}{\Chris{C}}
\newcommand{\vBtmp}{\tilde{\vB}}

\title{Numerical Methods to Compute the Coriolis Matrix and Christoffel Symbols for Rigid-Body Systems
\vspace{-15px}
}

\author{Sebastian Echeandia
    \affiliation{
	Aerospace and Mechanical Engineering\\
        University of Notre Dame\\
        Notre Dame, Indiana, 46556
        \vspace{-5px}
    }	
}

\author{Patrick M. Wensing\thanks{Corresponding Author} 
    \affiliation{Assistant Professor, Member of ASME\\
        Aerospace and Mechanical Engineering\\
        University of Notre Dame\\
        Notre Dame, Indiana, 46556\\
         Email: pwensing@nd.edu
         \vspace{-5px}
    }
}

\usepackage{mathtools}
\usepackage{cmll}
\usepackage{cite}

\usepackage{subcaption}
\usepackage{graphicx}

\usepackage{algorithm}
\usepackage{algorithmic}

\usepackage{amsthm}
\usepackage{setspace}
\allowdisplaybreaks

\begin{document}
\pagestyle{plain}

\maketitle

\newtheorem{proposition}{Proposition}
\newtheorem{definition}{Definition}
\newtheorem{corollary}{Corollary}
\newtheorem{theorem}{Theorem}
\newtheorem{remark}{Remark}
\newtheorem{lemma}{Lemma}

\newtheorem{property}{Property}

\newtheorem{note}{Note}

\setlength{\abovedisplayskip}{10pt plus0pt minus2pt}
\setlength{\abovedisplayshortskip}{\abovedisplayskip} 
\setlength{\belowdisplayshortskip}{\abovedisplayskip}
\setlength{\belowdisplayskip}{\abovedisplayskip}

\DeclarePairedDelimiter{\ceil}{\lceil}{\rceil}

\begin{abstract}
{\it This article presents  methods to efficiently compute the Coriolis matrix and underlying Christoffel symbols (of the first kind) for tree-structure rigid-body systems. The algorithms can be executed purely numerically, without requiring partial derivatives as in unscalable symbolic techniques. The computations share a recursive structure in common with classical methods such as the Composite-Rigid-Body Algorithm and are of the lowest possible order: $O(Nd)$ for the Coriolis matrix and $O(Nd^2)$ for the Christoffel symbols, where $N$ is the number of bodies and $d$ is the depth of the kinematic tree. Implementation in C/C++ shows computation times on the order of 10-20 $\mu$s for the  Coriolis matrix and 40-120 $\mu$s for the Christoffel symbols on systems with 20 degrees of freedom. The results demonstrate feasibility for the adoption of these algorithms within high-rate ($>$1kHz) loops for model-based control applications.}
\end{abstract}


\vspace{20px}
\section{Introduction}

Rigid-body dynamics algorithms have evolved to be an important component of many model-based robot control strategies. Most commonly, algorithms focus on computing the inverse dynamics of a system, or components of its equations of motion \cite{Featherstone08,Walker82} (e.g., the mass matrix, generalized Coriolis force, generalized gravity force, etc.). These results are then used to select joint torques that achieve desired movements or contact forces, with applications from computed torque control of manipulators to whole-body control of legged robots \cite{Abe07,Park07,Wensing13,Kuindersma15}. More recently, interest has increased toward computing other components of the equations of motion for application in disturbance observer problems \cite{DeLuca09,bledt2018contact} or the numerical calculation of partial derivatives for application to gradient-based motion optimization \cite{Bobrow01,JainRodriguez93,Suleiman08}.

This paper derives a new algorithm for the calculation of the Coriolis matrix.\footnote{Post-publication note: The authors wish to acknowledge a fundamentally similar version of the algorithm that was brought to their attention in January 2022. Algorithm 1 herein was originally documented privately in 2014 \cite{WensingNotes14} and 2017 \cite{wensingweb}, and was published in 2021. A closely related Coriolis matrix algorithm is available in Pinocchio \cite{carpentier2019pinocchio}. This independently derived method was developed in 2017 and released publicly in 2018 \cite{pinocchioweb}, but is not otherwise documented in the literature. The algorithms have the same conceptual operation. Footnote 2 provides additional detail.\\[.5ex] }  The algorithm mirrors well-established efficient algorithms for computing the mass matrix \cite{Walker82,Featherstone08}, and builds from formulas appearing in the adaptive control literature \cite{Lin95,Wang10,WangH12}. 
 From this Coriolis matrix algorithm, we also derive a new method for calculating the Christoffel symbols of the first kind. The final algorithm can be run purely numerically and does not require any symbolic partial derivatives.  A related algorithm was recently proposed in \cite{Safeea19} that is applicable for kinematic chains with revolute joints. By adopting a coordinate-free development, the resulting algorithms in this paper have additional generality in terms of application to branched systems or with more general joint models (e.g., prismatic joints). These numerical methods for the Christoffel symbols could have broad relevance for geometric control algorithms, the calculation of second-order partial derivatives of the inverse dynamics model, or in other geometric methods. 

The equations of motion of a rigid-body system can be written as
\begin{equation}
\H(\vq) \vqdd + \C(\q,\qd)\qd + \g(\q) = \btau\,,
\label{eq:eom}
\end{equation}
where $\q\in \Rn$ are the generalized coordinates, $\vH(\vq)\in\Rnn$ is the mass matrix, $\vC(\vq,\vqd)\in \Rnn$ a Coriolis matrix, $\vg(\vq) \in \Rn$ the generalized gravity force, and $\btau\in\Rn$ the generalized applied force (often simply the vector of actuator torques for a robot with revolute joints). It is well known that there are many possible choices of $\vC(\vq,\vqd)$ that provide the correct dynamics. All such choices satisfy $\vqd\T{\big[} \Hd(\q,\qd) -2\C(\q,\qd) {\big]}\qd = 0$ \cite{siciliano2010robotics}, while many satisfy
\begin{equation}
\boldsymbol{\eta}\T {\big[} \Hd(\q,\qd) -2\C(\q,\qd) {\big]}\boldsymbol{\eta} = 0 \quad \forall,\vq,\vqd,\boldsymbol{\eta}\in\Rn\,.
\label{eq:skew_property}
\end{equation}
This condition \eqref{eq:skew_property} is the same as requiring $\Hd-2\C$ to be skew symmetric, or equivalently that $\Hd = \C+\C\T$. A special choice that satisfies \eqref{eq:skew_property}, denoted $\vCc$, is given by:
\begin{align}
{\big[}\vCc(\vq,\vqd){\big]}_{ij} &= \sum_k \Gamma_{ijk}(\q)\, \dot{q}_k\quad \textrm{, where} \label{eq:ChristoffelC}\\
\Gamma_{ijk}(\q) &= \frac{1}{2} \left[ \frac{\partial M_{ij}}{\partial q_k} + \frac{\partial M_{ik}}{\partial q_j} - \frac{\partial M_{jk}}{\partial q_i} \right]
\label{eq:GammaEq}
\end{align}
are the Christoffel symbols of the first kind \cite{siciliano2010robotics}. Note that the symbols satisfy $\Gamma_{ijk} = \Gamma_{ikj}$ since $\vH$ is symmetric. 

\begin{definition}[Christoffel-Consistent Coriolis Factorization]
Consider a system with mass matrix $\vH(\q)$ and the Coriolis matrix $\vCc(\q,\qd)$ given in \eqref{eq:ChristoffelC}. The matrix $\vCc(\q,\vqd)$ is named the Christoffel-consistent Coriolis factorization.
\end{definition}

\begin{definition}[Valid Coriolis Factorization] Consider a system and its Christoffel-consistent Coriolis factorization $\vCc(\q,\qd)$. A matrix valued function $\C(\q,\qd)$ is said to be a {\em valid Coriolis factorization} if for all $\vq,\vqd\in\Rn$
\[
\vCc(\vq,\vqd)\qd = \C(\vq,\vqd)\qd\,.
\]
\end{definition}
\noindent A valid Coriolis factorization is one that gives the correct equations of motion when used in \eqref{eq:eom}. Yet, the additional property \eqref{eq:skew_property} is commonly needed in control \cite{Slotine87} and contact detection \cite{DeLuca09,bledt2018contact},  motivating the following stricter definition.

\begin{definition}[Admissible Coriolis Factorization] Consider a system with mass matrix $\vH(\q)$. A matrix valued function $\C(\q,\qd)$ is said to be an {\em admissible factorization} if 
\begin{enumerate}
\item $\C(\vq,\vqd)$ is a valid factorization
\item $\forall \vq,\vqd\in\Rn$, $\Hd(\vq,\vqd)-2\C(\vq,\vqd)$ is skew symmetric.
\end{enumerate}
\label{def:AdmissibleCoriolisFactorization}
\end{definition}

Although the definition for $\vCc$ in \eqref{eq:ChristoffelC} appears in many robotics textbooks  (e.g., \cite{siciliano2010robotics,lynch2017modern}), it is challenging to compute for complex systems, as the symbolic computation of $\vH$ becomes burdensome for systems with many degrees of freedom (DoF), and symbolic differentiation to form the Christoffels symbols via \eqref{eq:GammaEq} does not scale well with $n$. In this paper, we consider tailored numerical methods to compute admissible factorizations $\vC$, the Christoffel-consistent factorization $\vCc$, and the Christoffel symbols $\Gamma_{ijk}$ by taking advantage of the underlying structure of \eqref{eq:eom} for open-chain rigid-body systems. Since gravity does not affect the Coriolis terms, it is ignored in the remainder of the paper.

\section{Conventions and Notation}
This section introduces background material for modeling the dynamics of rigid-body systems. 
We first review vector space fundamentals before setting notation for describing kinematic connectivity. Spatial vector algebra \cite{Featherstone08} is then presented for kinematic and dynamic analysis. We note that spatial vector algebra is conceptually equivalent to a Lie-theoretic treatment of rigid-body dynamics \cite{Park95}, with their relationship detailed in \cite{Traversaro16b}.

\vspace{-5px}
\subsection{Vector Space Fundamentals}

Consider two vector spaces $\V$ and $\W$ and denote by $L(\V,\W)$ the vector space of linear operators from $\V$ to $\W$. The dual vector space $\Vstar$ is the set of linear functionals on $\V$, i.e., $\Vstar = L(\V,\R)$. For any $\vy \in \V^*$ and any $\mathbf{x} \in \V$, we denote the evaluation of the functional $\vy(\vx)$ as $\vy \bullet \vx$. 
Given any operator $\A \in L(\V,\W)$, denote by $\A^* \in L(\Wstar, \Vstar)$ the adjoint of $\A$, which is the unique operator satisfying
\[
\z \bullet [\A \x] = [\A^*\, \z] \bullet \x \quad \forall \z \in \Wstar, \x \in \V\,.
\]  
In the case when $\W = \V^*$, if $\A = \A^*$ then $\A$ is said to be self-adjoint. When bases and associated dual bases are adopted for $\V$, $\W$, $\Vstar$, and $\Wstar$, the matrix representation of $\A^*$ coincides with the transpose of the representation of $\A$. In this case, $\A$ is self-adjoint if and only if its matrix representation is symmetric.

\vspace{-5px}
\subsection{Modeling Connectivity}
A rigid-body system can be modeled as a set of $N$ bodies connected by a set of joints, each with up to six DoF. The topology of these connections can generally be described by a connectivity graph. Here, we restrict ourselves to the consideration of rigid-body trees and denote $d$ as the depth of the tree. Bodies are numbered from 1 through $N$ such that body $i$'s predecessor $p(i)$ toward the root is less than $i$. These parent/child relationships induce a partial order on the set $\{1,\ldots,N\}$, which is denoted using  a binary relation ``$\preceq$''. We say $j\preceq i$ if body $j$ is in the path from body $i$ to the root of the tree. In this case, $j$ is said to be an ancestor of $i$. If $i\succeq j$ or $j \succeq i$ then we say $i$ and $j$ are related, and denote this relationship by $i\sim j$.
\begin{figure}[tbp]
\centerline{\includegraphics[width=.45\columnwidth]{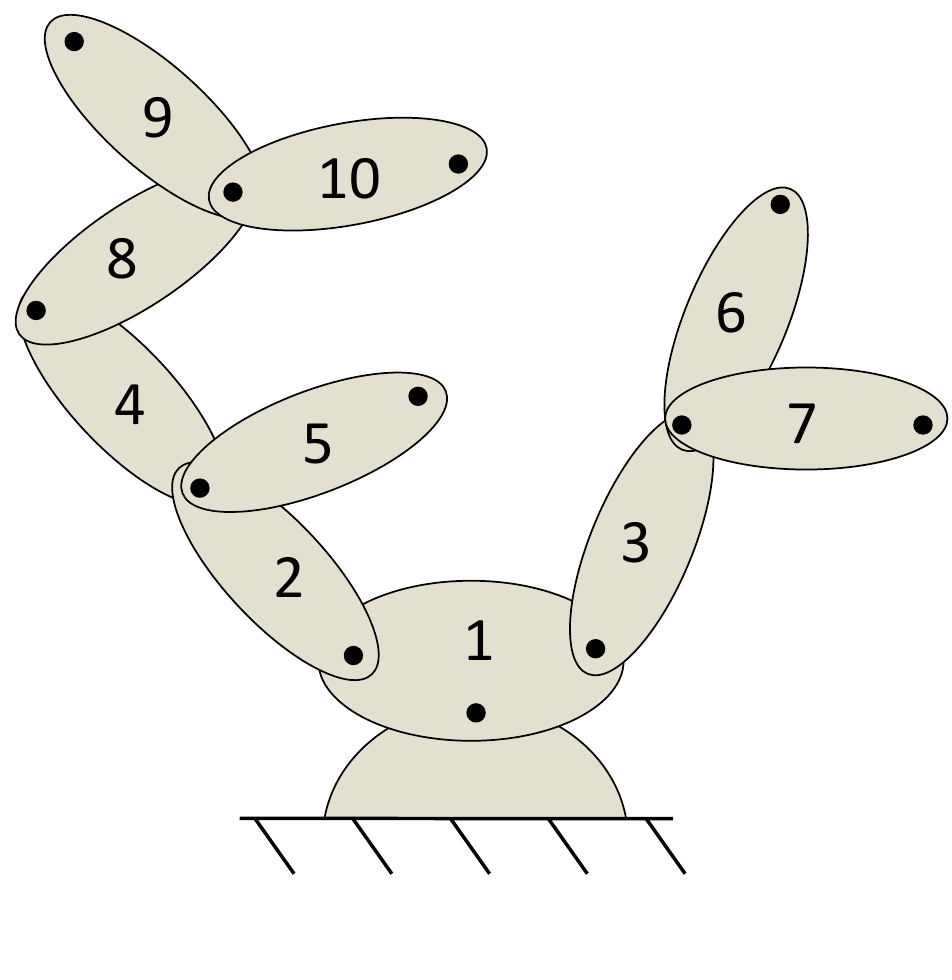}}
\vspace{-13px}
\caption{Sample body numbering for a branched system.}
\vspace{-10px}
\label{fig:support1}
\end{figure}

For any pair of relatives $i\sim j$, we denote $\ceil{ij}$ to be shorthand for the body closest to the leaves:
\[
\ceil{ij} = \begin{cases} 
i & \textrm{if~} i \after{j}  \\
j & \textrm{if~} j \after{i} \,.
        \end{cases}
\]
The $\ceil{ij}$ notation is best understood with an example. For the system in Figure~\ref{fig:support1} when $i=2$ and $j=10$ we have that $\ceil{ij}=10$. In contrast, $\ceil{ij}$ is undefined when $i=9$ and $j=5$ because neither body is a descendant of the other.

\subsection{Spatial Vector Algebra}
\label{sec:SVA}

This section presents spatial vector algebra for the development of algorithms in a coordinate-free manner. Importantly, this approach enables us to take derivatives of objects (vectors, inertias, etc.) in a coordinate-free sense, without needing to treat moving (rotating and/or translating) coordinate systems during algorithm derivation. The use of coordinates will then be an implementation detail handled after the main development. 

\vspace{3px}
\noindent {\bf Spatial Vectors:} The set of all rigid-body motion vectors forms a 6D vector space, denoted $\M$ \cite{Featherstone08}. Consider a body moving with spatial velocity $\bv\in\M$ and a Cartesian frame $A$. In coordinates, the expression of $\bv$ in frame $A$ is denoted
\[
{}^A\{\bv\} = \begin{bmatrix} {}^A \{ \boldsymbol{\omega} \} \\ {}^A \{ \vv_A \} \end{bmatrix}\,,
\]
where ${}^A\{\boldsymbol{\omega}\}$ denotes the angular velocity of the body and ${}^A\{\vv_A\}$ the linear velocity of the body-fixed point at the origin of $A$. Curly brackets $\{\cdot\}$ denote expression in coordinates, with the pre-superscript indicating the frame used. 

Suppose that body $i$ and its predecessor $p(i)$ are connected by a $d_i$ DoF joint with joint rates $\vqd_i \in \R^{d_i}$. Then, the spatial velocities of these bodies are related by \cite{Featherstone08}
\begin{equation}
\bv_i = \bv_{p(i)} + \vPhi_i \dot{\q}_i\,,
\label{eq:velocity_prop}
\end{equation}
where $\v_i,\v_{p(i)} \in \M$ represent the spatial velocities of bodies $i$ and $p(i)$, and  $\vPhi_i\in L( \R^{d_i}, \M)$ maps joint rates to joint velocities for the $i$-th joint. Collecting the configuration of each of the joints, generalized coordinates are chosen as
\begin{equation}
\q = [\q_1\T, \ldots, \q_{N}\T]\T\,.
\label{eq:gencoords}
\end{equation}

Spatial force vectors (i.e., force/moment pairs) are dual to spatial motion vectors, such that the vector space of spatial forces $\F$ is identified with $\M^*$. Given a body moving with velocity $\v\in\M$, the power delivered by a force $\f \in\F$ on the body is denoted by the ``dot product'' $\v\bullet \f$. 
Spatial inertias map motion vectors to force vectors and reside in a 10-dimensional subspace $\I \subset L(\M,\F)$. All inertias are self-adjoint, as their expression in coordinates is symmetric. See Appendix A for further detail.

\vspace{3px}
\noindent {\bf Spatial Equation of Motion:} The spatial equation of motion for each rigid body is given by
\begin{equation}
\Bf_i = \vI_i\, \ba_i + \bv_i \crf \vI_i \bv_i\, \label{eq:bodyForce}\,,
\end{equation}
where $\Bf_i \in \F$ is the net spatial force on body $i$, $\bv_i \in \Mot$ is its spatial velocity, $\ba_i \in \Mot$ its spatial acceleration, $\vI_i\in \I$ its spatial inertia, and $\times^*:\Mot\times\For\rightarrow\For$ the bi-linear cross-product operator between spatial motion vectors and spatial force vectors \cite{Featherstone08}. From an intuitive standpoint, the cross product $\bv \times^* \Bf$ gives the rate of change in $\Bf$ when any force field representing it moves with spatial velocity $\bv$. This spatial cross product generalizes the Cartesian formula $\dot{\mathbf{r}} = \boldsymbol{\omega} \times \mathbf{r}$ that describes the rate of change of a 3D vector $\mathbf{r}$ when rotating with angular velocity $\boldsymbol{\omega}$. Note that \eqref{eq:bodyForce} holds in coordinates for any frame, including when the frame origin does not coincide with the body's center of mass. 

\vspace{3px}
\noindent {\bf Cross Products:} Given any spatial velocity $\bv \in \Mot$, the motion/force cross product $\v\times^*\Bf$ can be used to define a unique linear operator $(\bv\times^*) \in L(\F,\F)$ such that $(\bv\times^*)\Bf = \bv \times^* \Bf$ for any force vector $\Bf$. 
We swap the order of the cross product arguments and denote $(\Bf \crff) \in L(\M,\F)$ as the unique linear operator satisfying
\beq
(\Bf \crff) \bv = (\bv\crf)\Bf\,. \nonumber
\eeq
See Appendix A for an expression in coordinates. Finally, we denote $(\bv \times)\in L(\M,\M)$ as the spatial motion/motion cross product defined from the adjoint of $(\bv\times^*)$ according to $(\v\crm) = - (\v \crf)^*$. That is, for any $\v,\w\in\M$ and $\f\in\F\!\!$
\beq
\left[ (\v \crm) \w\right ]\bullet \f   =  \w \bullet \left[ -(\v\crf) \f\right]
\label{eq:motionAndForceCrossRelationship}
\eeq
Since the adjoint corresponds to a transpose in coordinates, we adopt $\mathbf{A}{}\T$ in place of $\mathbf{A}^*$ as a matter of notation.

\vspace{3px}
\noindent {\bf Factorization of the Spatial Equation:} The bi-linear velocity-product term $\bv_i \crf \vI_i \bv_i$ in \eqref{eq:bodyForce} can be factorized in a variety of ways to take the form
\beq
\Bf_i = \vI_i \, \ba_i + \vB(\bv_i,\vI_i) \, \bv_i\,, \nonumber
\eeq
where $\vB(\bv_i, \vI_i) \in L(\M, \F)$. Since $\M^*=\F$, it follows that $\vB(\bv_i,\vI_i)\T \in L(\M,\F)$ as well. One immediate factorization can be taken as
\beq
\label{eq:simpleB}
\vB(\bv_i,\vI_i) = (\bv_i\times^*) \, \vI_i\,,
\eeq
while another proposed by Niemeyer and Slotine \cite{Niemeyer91} is
\beq
\vB( \bv_i, \vI_i) = \frac{1}{2}\left( (\bv_i \times^*) \vI_i + (\vI_i \bv_i  \crff) - \vI_i  (\bv_i \times)    \right)\,.
\label{eq:GunterB}
\eeq
Viewing $\vB( \cdot, \cdot )$ as a function from $\M\times\I$ to $L(\M,\F)$, the definitions above are bi-linear in their arguments.

Both of these factorizations
of the body-level velocity-product terms 
have additional properties that will be of interest for factorization of the system-level Coriolis terms $\vC(\vq,\vqd)\vqd$. Note that the rate of change in the spatial inertia of body $i$ is given in a coordinate-free sense as \cite{Featherstone08}
$\dot{\vI}_i =  (\bv_i \times^*) \vI_i - \vI_i (\bv_i \times)\,.$
 
Letting $\vB_i = \vB(\bv_i,\vI_i)$, factorizations \eqref{eq:simpleB} and \eqref{eq:GunterB} satisfy
\[
\bv \bullet {\big[}  (\dot{\vI}_i - 2 \vB_i) \bv {\big]} = 0 \quad \forall \bv\in \M\,,
\]
which is equivalent to the condition that the matrix representation of $\dot{\vI}_i - 2 \vB_i$ is skew-symmetric, or equivalently that
\beq
\dot{\vI}_i = \vB_i + \vB_i\T\,.
\label{eq:IdotCondition}
\eeq

\begin{definition}[Admissible Body-Level Factorization]
\label{def:admissible_body}
A function $\mathbf{B}(\cdot,\cdot): \M \times \I \rightarrow L(\M, \F)$ is said to be an {\em admissible body-level factorization} if~  $\forall~ \bv\in\M, \vI \in \mathcal{I}$
\begin{enumerate}
\item $ \vB(\bv,\mathbf{I})\bv = \bv\times^* \vI \v$ and 
\item $ (\bv\times^*) \vI - \vI (\bv \times)= \vB(\bv,\vI) + \vB(\bv,\vI)\T$\,.
\end{enumerate}
\end{definition}

\vspace{-1ex}
\section{A Numerical Method for Computing $\vC$}

This section presents a new numerical method to compute the Coriolis matrix whereby adopting an admissible body-level factorization $\vB$ leads to an admissible system-level factorization $\vC$. Niemeyer and Slotine \cite{Slotine88
,Niemeyer90,Niemeyer91} were the first to establish a link between the factorization of body-level terms in the Recursive-Newton-Euler Algorithm (RNEA) and admissibility of the Coriolis matrix.  Lin et al.~\cite{Lin95} later provided general conditions on body-level factorizations that result in admissible $\vC$. Others have taken similar strategies \cite{Ploen99,DeLuca09,Wang10,WangH12} for problems in sensorless contact detection, passivity-based control, or adaptive control. DeLuca and Ferrajoli \cite{DeLuca09} are the only ones to provide an algorithm to compute $\vC$.
 Their approach requires $N$ calls to a modified RNEA, with a total computation complexity of $O(N^2)$. We instead introduce a tailored method for $\vC$ and reduce the computation complexity down to $O(Nd)$, enabling savings for branched systems.

\vspace{-1ex}
\subsection{Factoring the RNEA}
The RNEA computes the inverse dynamics of a system with two passes over its kinematic tree. The first pass moves outward from base to tips, recursively computing the spatial velocity and acceleration of each body \cite{Featherstone08} with \eqref{eq:velocity_prop} and
\begin{align}
\va_i = \va_{p(i)} + \vPhi_i \vqdd_i + \vPhid_i \vqd_i \label{eq:acc_rec}\,,
\end{align}
where $\vPhid_i = (\v_i \crm) \vPhi_i + \mathring{\vPhi}_i$. The term  $(\v_i \crm) \vPhi_i$ gives the derivative due to the joint moving, and $\mathring{\vPhi}_i$ the derivative due to the joint axes changing in local coordinates, defined by
\begin{equation}
{\vphantom{ \mathring{\vPhi}_i} }^i
{\big\{} \mathring{\vPhi}_i {\big \}} = \ddt {}^i \{  \vPhi_i \}\,. \nonumber
\end{equation}
For revolute joints with fixed axes, $\mathring{\vPhi}_i=\mathbf{0}$.

With this information, the net force on each body can be computed using \eqref{eq:bodyForce}. The RNEA backwards pass then sums up these forces over descendants, computing joint torques
\beq
\btau_i = \vPhi_i\T \sum_{\sumafter{k}{i}} \Bf_k = \vPhi_i\T \sum_{\sumafter{k}{i}} \left[ \vI_k \va_k + \bv_k \times^* \vI_k \bv_k\right]\,. \label{eq:RNEA}
\eeq
Using the derivation in Appendix B, \eqref{eq:RNEA} is reorganized as
\begin{align}
\btau_i &= \sum_{\sumrel{j}{i}} \vPhi_i\T \vI_{\ceil{ij}}^C \vPhi_j \vqdd_j + ( \vPhi_i \T \vI_{\ceil{ij}}^C \vPhid_j + \vPhi_i\T \vB_{\ceil{ij}}^C \vPhi_j )\, \vqd_j 
	\label{eq:ForExtraDerivation}
\end{align}
where composite quantities are defined by
\begin{align}
\vI_{\ceil{ij}}^C &= \sum_{\sumafter{k}{\ceil{ij}}} \vI_k \quad \textrm{and}\\
\vB_{\ceil{ij}}^C &= \sum_{\sumafter{k}{\ceil{ij}}} \vB_k( \bv_k, \vI_k)\,.
\label{eq:Bcomposite}
\end{align}
The composite inertia $\vI_j^C$ represents the total inertia of body $j$ and its descendants. Likewise, $\vB_j^C$ accounts for Coriolis and centripetal forces for body $j$ and its descendants. The mass matrix and a valid Coriolis matrix are then given by
\begin{align}
\vH_{ij} &= \vPhi_i\T \vI_{\ceil{ij}}^C \vPhi_j \quad \textrm{and} \\[.5ex]
\vC_{ij} &= \vPhi_i \T \vI_{\ceil{ij}}^C \vPhid_j + \vPhi_i\T \vB_{\ceil{ij}}^C \vPhi_j \label{eq:Cij_general}
\end{align}
when $i\sim j$, and  $\vH_{ij} = \vC_{ij} = \bzero$ otherwise.


\subsection{Recursive Algorithm for Computing $\C$}
Toward simplifying these expressions, when $i \before{j}$
\begin{align}
\vH_{ij} &= \vPhi_i \T \, \vI_j^C \, \vPhi_j\,, \label{eq:Hij}\\[.5ex]
\vC_{ij} &= \vPhi_i \T \, ( \vI_j^C \, \vPhid_j + \vB_j^C \, \vPhi_j) \, \textrm{,~and}
\label{eq:Cij}\\[.5ex]
(\vC_{ji})\T &= \vPhid{}_i\T \, \vI_j^C \, \vPhi_j + \vPhi_i\T \, ( \vB_j^C)\T \, \vPhi_j\,. \label{eq:Cji}
\end{align}
To enable computing all of these quantities recursively, let
\begin{align*}
\BF_{1,j} &= \vI_j^C \, \vPhid_j + \vB_j^C \, \vPhi_j\,, \\
\BF_{2,j} &= \vI_j^C \, \vPhi_j\,,~\textrm{and} \\
\BF_{3,j} &=  ( \vB_j^C)\T \, \vPhi_j\,.
\end{align*}
All of these terms can be computed with complexity $O(N)$, since all $\vI_j^C$ and $\vB_j^C$ can be computed with $O(N)$ cost via summing backward along the tree. Considering body $j$, for all $O(d)$ ancestors $i \before{j}$
\begin{align*}
\vH_{ij} & = \vPhi_i\T \, \BF_{2,j}\,, \\
\vC_{ij} &= \vPhi_i\T \, \BF_{1,j}\,,~\textrm{and}\\
\vC_{ji}&= \left(\vPhid{}_i\T\BF_{2,j} + \vPhi_i\T \BF_{3,j}\right)\T\,.
\end{align*}

These equations enable an $O(Nd)$ method to compute $\vH$ and $\vC$ as given in Algo.~\ref{alg:cor}. Since the algorithm runs on a computer, vectors and matrices are expressed in coordinates at each step. 
All quantities for body $i$ are expressed using a local body-fixed frame, and spatial transformation matrices $\XM{i}{p(i)}$ are used to transform quantities between frames.
For example, in coordinates, \eqref{eq:velocity_prop} takes the form
\[
{}^i\{\bv_i\} = \XM{i}{p(i)} {}^{p(i)}\{ \bv_{p(i)} \} + {}^i\left\{ \vPhi_i\right\} \vqd_i
\]
such that the matrix $\XM{i}{p(i)}$ changes the basis of a motion vector from frame $p(i)$ to frame $i$. Likewise, the matrix $\XM{i}{p(i)}\T$ provides a change of basis for spatial force vectors from frame $i$ to frame $p(i)$. 
Finally, for any operator $\mathbf{I} \in L(\M,\F)$, the congruence transform ${}^{p(i)}\{ \vI \} = \XM{i}{p(i)}\T  {}^i\{ \vI \}  \XM{i}{p(i)}$ changes the representation from frame $i$ to frame $p(i)$. This transform is often used with inertias \cite{Featherstone08}, but also provides the transform law for the body-level factorization terms $\vB_i$ and their associated composite quantities $\vB_i^C$. For cleanliness of presentation, the bracket notation $\{\cdot\}$ and specification of frames are omitted in the algorithm.

The structure of Algorithm \ref{alg:cor} is as follows. A forward sweep (lines 2-7) computes the velocity of each body and the initial composite terms $\vI_i^C$ and $\vB_i^C$. Other valid body-level factorizations may be used on line 6. The backward sweeps (lines 8-26) compute the entries of the Coriolis and mass matrices. 
The while loop (lines 16-23) computes entries of $\vH$, $\Hd$, and $\vC$ associated with body $j$ and propagates the computation down to all its predecessors.\footnote{Post-publication note: By comparison, the Coriolis matrix algorithm in Pinocchio \cite{carpentier2019pinocchio} expresses all spatial quantities in ground coordinates to limit coordinate transformations and improve efficiency. Its algorithm \cite{pinocchioweb} originally used the equivalent of $\vB_i^C = \bv_i \crf \vI_i$ on Line 6 herein. That algorithm was amended following the publication of this work to use the body-level factorization $\frac{1}{2}[ (\bv_i \times^*) \vI_i + (\vI_i \bv_i)\crff-   \vI_i (\bv_i \times )]$ and provide the Christoffel-consistent Coriolis matrix \cite{pinocchioEdit}. This change makes the algorithms effectively identical.\\[-.5ex]} 

\begin{algorithm}[t]
\caption{Coriolis Matrix Algorithm}
\setstretch{1.2}
\begin{algorithmic}[1]
\REQUIRE $\vq,\,\vqd$
\STATE $\bv_0 = \bzero$
\FOR{$i=1$ to $N$}
\STATE $\bv_i = \XM{i}{p(i)}\,\bv_{p(i)}+\vPhi_i \,\vqd_i$ 
\STATE $ \vPhid_i = (\bv_i \times) \vPhi_i + \mathring{\vPhi}_i$ 
\STATE $\vI_i^C = \vI_i$ 
\STATE $\vB_i^C = \frac{1}{2}[ (\bv_i \times^*) \vI_i + (\vI_i \bv_i)\crff-   \vI_i (\bv_i \times )]$ 
\ENDFOR
\FOR{$j=N$ to $1$}
\STATE $\BF_1 = \vI_j^C\, \vPhiDot_j + \vB_j^C \, \vPhi_j  $ \STATE $\BF_2 = \vI_j^C\, \vPhi_j$ 
\STATE $\BF_3 = (\vB_j^C)\T \vPhi_j$ 
\STATE $\vC_{jj} = \vPhi{}_j\T\, \BF_1$ 
\STATE $\vH_{jj} = \vPhi{}_j\T \, \BF_2$
\STATE $\dot{\vH}_{jj} = \dot{\vPhi}{}_j\T\, \BF_2 + \vPhi_j\T \, (\BF_1 + \BF_3) $
\STATE $i=j$
\WHILE{$p(i)>0$}
\STATE $\BF_1\!=\!\XM{i}{p(i)}\T\,\BF_1$;\,$\BF_2\!=\!\XM{i}{p(i)}\T\,\BF_2$;\, $\BF_3\!=\!\XM{i}{p(i)}\T\,\BF_3$
\STATE $i = p(i)$ 
\STATE $\vC_{ij} = \vPhi{}_i\T\, \BF_1$ 
\STATE $\vC_{ji} = (\vPhiDot\vphantom{\vPhi}_i\T\, \BF_2 + \vPhi{}_i\T \vF_3 )\T$
\STATE $\vH_{ij} = (\vH_{ji})\T = \vPhi{}_i\T\, \BF_2$ 
\STATE $\dot{\vH}_{ij} = (\dot{\vH}_{ji})\T = \vPhid{}_i\T\, \BF_2  + \vPhi{}_i\T(\BF_1+\BF_3)$ 
\ENDWHILE
\STATE $\vI_{p(j)}^C = \vI_{p(j)}^C + \XM{j}{p(j)}\T \, \vI_j^C \, \XM{j}{p(j)} $\\[.5ex]
\STATE $\vB_{p(j)}^C = \vB_{p(j)}^C + \XM{j}{p(j)}\T \, \vB_j^C \, \XM{j}{p(j)} $
\ENDFOR
\RETURN $\vH$, $\dot{\vH}$, $\vC$
\end{algorithmic}
\label{alg:cor}
\end{algorithm}

\begin{proposition}[Algorithm for an Admissible Coriolis Factorization]
\label{prop:admissible}
Suppose that Algo.~1 uses an admissible body-level factorization $\vB(\bv,\vI)$ on Line 6. Then, the resulting $\C(\q,\qd)$ is an admissible Coriolis factorization. 
\end{proposition}

\begin{proof} To show that Algo.~\ref{alg:cor} gives an admissible factorization $\vC$, it remains to show that $\dot{\H} = \vC + \vC\T$. From \eqref{eq:Hij}
\[
\dot{\vH}_{ij} = \vPhid{}_i \T \vI_j^C \vPhi_j + \vPhi_i \T \left( \dot{\vI}_j^C \vPhi_j + \vI_j^C \vPhid_j \right)\,.
\]
For an admissible body-level factorization, \eqref{eq:IdotCondition} holds, and, correspondingly, $\dot{\vI}_j^C = \vB_j^C + (\vB_j^C)\T$. It then follows that
\[
\dot{\vH}_{ij} = \vPhi_i\T( \BF_{1,j} + \BF_{3,j}) + \vPhid{}_i\T\BF_{2,j} = \vC_{ij}+\left(\vC_{ji}\right)\T\\[-3.5ex]
\]
\end{proof}

\begin{remark}
This proof is conceptually equivalent to \cite[pp.~2325]{Lin95}, with \eqref{eq:Cij_general} sharing the same form as (19c) in \cite{Lin95}. Via comparison, our coverage of kinematic branching is unique, but the main novelty is the algorithm for $\vC$. 
\end{remark}

\begin{proposition}[Algorithm for the Christoffel-Consistent Factorization]
Consider a kinematic tree where each joint is single DoF and satisfies $\mathring{\vPhi}_i = \mathbf{0}$. Suppose Algo.~\ref{alg:cor} uses the body-level factorization given on Line 6. Then, Algo.~\ref{alg:cor} returns the Christoffel-consistent factorization $\vCc$.
\label{prop:Cc}
\end{proposition}

\begin{proof}
A formula for $\vCc_{ij}$ is given in \cite[Eq.~(3.19)]{Niemeyer90} for serial chains, and generalizes to branched trees by replacing integer ordering $i\le j$ with the partial order $i \preceq j$. The original derivation is lengthy and so its generalization is omitted here. Using the factorization on Line 6 of Algo.~\ref{alg:cor}, it can be verified that \eqref{eq:Cij_general}  matches \cite[Eq.~(3.19)]{Niemeyer90} . 
\end{proof}

\begin{remark}
\label{rem:GunterCompare}
Note that \cite{Niemeyer90} provides an algorithm to compute $\vCc(\vq,\vqd) \vqd_r$ where $\vqd_r\in \R^n$ is a reference joint velocity not necessarily equal to $\vqd$. While we use the formula for $\vCc_{ij}$ in \cite{Niemeyer90} as the basis for our proof, a distinction of our work is the ability to compute $\vCc$ itself.
\end{remark}

\newcommand{\qh}{\changeCoord{\q}}
\newcommand{\qhd}{\dot{\qh}}
\subsection{Changes of Generalized Coordinates}

Algorithm~\ref{alg:cor} only applies when joint variables \eqref{eq:gencoords} are used for the generalized coordinates. A change of coordinates can be applied to the output of Algo.~\ref{alg:cor} in other cases. Let $\changeCoord{\q}$ represent a different choice of coordinates with 
\begin{equation}
\A = \frac{\partial \q}{\partial \vqh} \text{~~~such~that~~~}  A_{ij} = \pq{i}{j}\,.
\label{eq:coordChange}
\end{equation}
Using $\vqd = \vA \dot{\qh}$ and $\vqdd = \A \ddot{\qh} +\dot{\A} \qhd$ in \eqref{eq:eom}, and multiplying both sides of \eqref{eq:eom} by $\A\T$, it follows that the quantities
\begin{align}
\Hhat &= \A\T \H \A\,, \label{eq:Htransform} \\
\vChat &= \A\T \C \A + \A\T \H \dot{\A}\,,  
\label{eq:Ctransform} \\
\changeCoord{\g} &= \A\T \g\,,\textrm{~and}  \label{eq:gtransform}\\
\changeCoord{\btau} &= \A\T \btau \label{eq:tautransform}
\end{align}
lead to transformed equations of motion
\begin{equation}
\Hhat \, \ddot{\qh} + \vChat\, \qhd + \changeCoord{\g} = 
\changeCoord{\btau}\,.\nonumber
\end{equation}

\begin{proposition}[Admissible Factorization Under a Change of Coordinates \cite{Lin95}]
\label{prop:AdmissibleTransform}
 If the matrix $\C(\vq,\vqd)$ is an admissible factorization for the coordinates $\q$, then $\vChat(\qh, \qhd)$ given by \eqref{eq:Ctransform} is an admissible factorization for the coordinates $\qh$.
\end{proposition}

A remarkable property of the transformation law \eqref{eq:Ctransform} is that it transforms the unique Christoffel-consistent factorization in one set of coordinates to the unique Christoffel-consistent factorization in the transformed coordinates.
\begin{theorem}[Christoffel-Consistent Factorization under a Change of Coordinates]
\label{thm:CcTransform}
Suppose $\vCc(\vq,\vqd)$ is the unique factorization given by the Christoffel symbols in the coordinates $\q$ via \eqref{eq:ChristoffelC}. Consider a change of coordinates to $\qh$ with $\A$ defined as in \eqref{eq:coordChange}. Then, the unique Coriolis factorization given by the Christoffel symbols in the coordinates $\qh$ is
\begin{equation}
\vCchat = \A{}\T \vCc \A + \A{}\T \H \dot{\A}\,.
\label{eq:TransformedCc}
\end{equation}
\end{theorem}
\begin{proof}
The transformation law for Christoffel symbols is
\begin{align}
\Gammahat_{ijk} = \sum_{\alpha,\beta, \gamma} \pq{\alpha}{i} \pq{\beta}{j} \pq{\gamma}{k} \Gamma_{\alpha \beta \gamma} + \sum_{\alpha, \beta} \pq{\alpha}{i} \frac{\partial^2 q_\beta}{\partial \changeCoord{q}_j \changeCoord{q}_k}  M_{\alpha \beta} \nonumber\,.
\end{align}
Multiplying both sides by $\dot{\changeCoord{q}}_k$ and summing over $k$ we have:
\begin{align*}
\Cchat_{ij} &= \sum_{\alpha, \beta, \gamma, k} \pq{\alpha}{i} \pq{\beta}{j} \left(\pq{\gamma}{k} \dot{\changeCoord{q}}_k \right) \Gamma_{\alpha \beta \gamma} \\
&\qquad \qquad \qquad
+\sum_{\alpha, \beta, k} \pq{\alpha}{i} \left( \frac{\partial^2 q_\beta}{\partial \changeCoord{q}_j \changeCoord{q}_k} \dot{\changeCoord{q}}_k \right)  M_{\alpha \beta} \\ 
&=\sum_{\alpha, \beta, \gamma} \pq{\alpha}{i} \pq{\beta}{j} \Gamma_{\alpha \beta \gamma} \dot{q}_\gamma + \sum_{\alpha, \beta} \pq{\alpha}{i} \left(\frac{\rm d}{{\rm d}t}  \pq{\beta}{j} \right)  M_{\alpha\beta} \\
&= \sum_{\alpha, \beta} A_{\alpha i} \Cc_{\alpha\beta} A_{\beta j} + A_{\alpha i} M_{\alpha \beta} \dot{A}_{\beta j}\\
&= \left[ \A\T \vCc \A + \A\T \H \dot{\A}\right]_{ij}\\[-8ex]
\end{align*}
\end{proof}

\subsection{Other Remarks}
\label{sec:remarks}

\begin{remark}
Due to the symmetry of the Christoffel symbols, any valid factorization is related to $\vCc$ via
\begin{equation}
\vCc(\vq,\vqd) = \frac{1}{2} \frac{\partial}{\partial \qd} \left[\vC(\vq,\qd) \qd \right]\,.
\label{eq:CfromPartials}
\end{equation}
Thus, considering \eqref{eq:eom} and \eqref{eq:CfromPartials}, Algo.~1 can be used to efficiently compute $\partial \btau / \partial \qd$ since $\partial \btau / \partial \qd = 2 \vCc$. 
\end{remark}

\begin{remark}
Proposition~\ref{prop:admissible} also has applicability for alternate methods of computing $\vC$ \cite{Garofalo13}. Consider body Jacobians $\vJ_k \in L(\R^n, \M)$ satisfying $\bv_k = \vJ_k\, \dot{\q}$. 
From \eqref{eq:RNEA}, it can be shown that
\[
\btau = \sum_k \vJ_k\T \vI_k \vJ_k \vqdd + \vJ_k\T\left[ \vB(\bv_k,\vI_k)\vJ_k + \vI_k \dot{\vJ}_k \right] \vqd\,.
\]
With any factorization in Algo.~\ref{alg:cor} Line 6, its outputs satisfy:
\begin{align}
\vH &= \sum_k \vJ_k\T \vI_k \vJ_k \quad \textrm{and}\\
\vC &= \sum_k \vJ_k\T \left[ \vB(\bv_k,\vI_k) \vJ_k + \vI_k \dot{\vJ}_k \right]\,.
\label{eq:CGlobal}
\end{align}
Here, each $\vJ_k$ is a linear mapping to coordinate-free velocities. Via comparison, in \cite{Garofalo13} ${}^k\{\vJ_k\} \in L(\R^n, \R^6)$ is used to express the velocity in coordinates. 
Defining $\mathring{\vJ}_k$ as the operator satisfying $^k\{\mathring{\vJ}_k\} = \ddt {}^k \{\vJ_k\}$ one then has that 
\begin{align*}
\dot{\vJ}_k = \mathring{\vJ}_k + (\bv_k \times) \vJ_k\,.
\end{align*}
Extending results in \cite{Garofalo13}, if $\vB(\cdot,\cdot)$ is admissible then so is
\begin{equation}
\vC = \sum_k \vJ_k\T \left[ \vB(\bv_k,\vI_k) \vJ_k + \vI_k \mathring{\vJ}_k + \vI_k (\bv_k\times) \vJ_k \right] \label{eq:gianluca}
\end{equation}
When used directly, however, this formula has cost $O(N^3)$.
\end{remark}

\newcommand{\vnu}{\boldsymbol{\nu}}
\begin{remark}
This paper has thus far required generalized coordinates. 
Suppose instead that generalized speeds $\boldsymbol{\nu} \in \Rn$ are adopted \cite{kane1985dynamics} with 
$\A(\q) = \frac{\partial{\qd}}{\partial \boldsymbol{\nu}}$.
Eqs.~\eqref{eq:Htransform}-\eqref{eq:tautransform} still lead to valid equations of motion and Prop.~\ref{prop:AdmissibleTransform} generalizes to this case. If one adopts an admissible body-level factorization and defines $\vJ_k$ such that $\bv_k = \vJ_k \,\vnu$, then \eqref{eq:CGlobal} and \eqref{eq:gianluca} also lead to an admissible Coriolis factorization. In the special case when generalized speeds $\boldsymbol{\nu}_i$ are adopted for each joint~$i$ (i.e., such that $\bv_i = \bv_{p(i)} + \vPhi_i \vnu_i$), Algo.~\ref{alg:cor} can be used directly and Prop.~\ref{prop:admissible} remains applicable as well. 
\end{remark}

\begin{remark}
 Consider a floating-base system such as a humanoid or quadrotor. The use of the CoM position $\mathbf{p}_{CoM} \in \R^3$ in the generalized coordinates leads to decoupled equations of motion \cite{garofalo2015inertially,garofalo2018task} in this case. Consider coordinates $\q = [ \mathbf{p}_{CoM}\T, \q_2\T ]\T$ where  $\q_2$ is independent of the CoM position. Then, the mass matrix $\vH$ takes a block-diagonal form \cite{garofalo2015inertially}, as does the Coriolis matrix $\vCc$  
\begin{equation}
\H\!=\!\begin{bmatrix} m \mathbf{1}_{3} ~&~ \mathbf{0} \\ \mathbf{0} ~&~ \H_{22}(\q_2) \end{bmatrix}~~~\vCc \!=\!\begin{bmatrix} \mathbf{0} ~&~ \mathbf{0} \\ \mathbf{0} ~&~ \vCc_{22}(\qh_2, \dot{\qh}_2) \end{bmatrix}\,,
\label{eq:sparseC}
\end{equation}
where $m \in \R$ is the system's mass. The sparsity of $\vCc$ is due to the fact that, $\Gamma_{ijk} = 0$ when $i$,$j$, or $k$ is 1, 2, or 3. 

For floating-base systems, it remains desirable to adopt generalized speeds so that the angular velocity of a main body can be used to avoid parametrization singularities. Consider $\boldsymbol{\nu} = \left[\dot{\mathbf{p}}_{CoM}\T,~ \boldsymbol{\nu}_2\T \right]\T$, where $\boldsymbol{\nu}_2$ completes the generalized speeds and is independent of the CoM velocity. Then, the Coriolis matrix transformation \eqref{eq:Ctransform} applied to $\vCc$ in \eqref{eq:sparseC} gives an admissible factorization with the same sparsity. Alternately, if \eqref{eq:CGlobal} or \eqref{eq:gianluca} is used with $\vB(\bv,\vI)$ from \eqref{eq:GunterB} and Jacobians satisfying $\bv_k = \vJ_k \vnu$, then the same sparse Coriolis matrix is obtained directly. See \cite{Mishra20} for a discussion of other structured factorizations.   
\end{remark}

\section{Computing Christoffel Symbols}

This section builds upon the previous one toward an algorithm for the Christoffel symbols of the first kind. We consider the case when all joints are single DoF and $\mathring{\vPhi}_i=0$ for all $i$. It is further assumed that $\vB(\bv,\vI)$ matches the factorization in \eqref{eq:GunterB}. In this case, Prop.~\ref{prop:Cc} ensures that Algo.~\ref{alg:cor} returns $\vCc$, with the formula for $\vCc_{ij}$ recalled as
\begin{equation}
\vCc_{ij} = \vPhi_i \T \, ( \vI_{\ceil{ij}}^C \, \vPhid_j + \vB_{\ceil{ij}}^C \, \vPhi_j)\,,
\label{eq:Cij_again}
\end{equation}
when $i\sim j$ and $\vCc_{ij}=0$ otherwise. This factorization serves as the starting point for the Christoffel symbols algorithm since \eqref{eq:ChristoffelC} implies $\Gamma_{ijk} = \nicefrac{\partial \vCc_{ij}}{\partial \dot{q}_k}$.

Since the formula for $\vCc_{ij}$ in \eqref{eq:Cij_again} includes effects from velocities $\v$, rates of change in joint axes $\vPhid$, and body-level factorizations $\vB$, we consider the partials of each of these terms individually before combining them together. Noting that $\bv_j = \sum_{\sumbefore{i}{j}} \vPhi_i \, \dot{q}_i$, it follows that
\begin{equation}
\frac{\partial \v_j}{\partial \dot{q}_k} = \begin{cases} \vPhi_k & \textrm{if $j \after{k}$} \\
                     \mathbf{0} & \textrm{otherwise.}
        \end{cases}
\label{eq:partial_v}
\end{equation}
Since $\vPhiDot_j = \bv_j \times \vPhi_j$ it then follows similarly that
\begin{align}
\frac{\partial \vPhiDot_j}{\partial \dot{q}_k} &= \begin{cases} \vPhi_k\times \vPhi_j & \textrm{if $j \after{k}$} \\
                     \mathbf{0} & \textrm{otherwise.}
        \end{cases}
\label{eq:dPhidq}
\end{align}

Moving to the next terms in \eqref{eq:Cij_again}, consider how $\vB_j^C$ changes with $\dot{q}_k$, and recall that
 $\vB(\bv, \vI)$ is bi-linear in its arguments. Combining this property with \eqref{eq:Bcomposite} and \eqref{eq:partial_v}
\begin{align*}
\frac{\partial \vB_{\ceil{ij}}^C}{\partial \dot{q}_k} &= \sum_{\ell \after{\ceil{ij}}} \vB\left( \frac{\partial \v_\ell}{\partial \dot{q}_k} , \vI_\ell \right)\\
&=\sum_{\ell \in \{ \ell \after{\ceil{ij}} \textrm{~and~} \ell \after{k} \} } \vB(\vPhi_k, \vI_\ell)\\[1ex]
 &= \begin{cases} \vB(\vPhi_k, \vI_{\ceil{ijk}}^C) & \textrm{if}~\ceil{ij} \sim k \\
 \bzero & \textrm{otherwise\,,}
 \end{cases}
\end{align*}
where $\ceil{ijk} = \ceil{ \ceil{ij} k}$ gives the body closest to the leaves for the mutual relatives $i$, $j$, and $k$.

With these elements, the symbol $\Gamma_{ijk}$ is constructed as
\[
\Gamma_{ijk}= \frac{\partial \vCc_{ij}}{\partial \dot{q}_k} = \vPhi_i\T \vI_{\ceil{ij}}^C \frac{\partial \vPhid_j}{\partial \dot{q}_k} + \vPhi_i\T \frac{\partial \vB_{\ceil{ij}}^C }{\partial \dot{q}_k} \vPhi_j\,.
\]
Without loss of generality, since $\Gamma_{ijk}=\Gamma_{ikj}$ we consider $j\preceq k$. 
It then follows that when $i$, $j$, and $k$ are mutually related (i.e., $i \sim j$, $j \sim k$, and $i\sim k$) then
\begin{equation}
\Gamma_{ijk} = \Gamma_{ikj} = \vPhi_i\T \vB(\vPhi_k, \vI_{\ceil{ijk}}^C) \vPhi_j
\label{eq:closed_form_symbols}
\end{equation}
 and $\Gamma_{ijk}=\Gamma_{ikj}=0$ otherwise. Alternate closed-form expressions for the symbols $\Gamma_{ijk}$ can be found in \cite{Brockett93,Park95,muller2003lie,Muller16,Wei94}.
 Considering the case when $i\preceq j \preceq k$, \eqref{eq:closed_form_symbols} gives
\begin{align}
\Gamma_{ijk} &= \Gamma_{ikj} = \vPhi_i\T \vB(\vPhi_k, \vI_k^C) \vPhi_j\,, \label{eq:chris1}\\
\Gamma_{jik} &= \Gamma_{jki} = \vPhi_j\T \vB(\vPhi_k, \vI_k^C) \vPhi_i\,, \textrm{~and} \label{eq:chris2} \\
\Gamma_{kij} &= \Gamma_{kji} = \vPhi_k\T \vB(\vPhi_j, \vI_k^C) \vPhi_i\,. \label{eq:chris3}
\end{align}
By using $\vB(\bv,\vI)$ from \eqref{eq:GunterB} and applying the cross-product property \eqref{eq:motionAndForceCrossRelationship}, 
\eqref{eq:chris3} can then be rearranged as
\begin{align}
\Gamma_{kij} &= \frac{1}{2} \vPhi_i\T \!\left[  \vI_k^C (\vPhi_k \times) - (\vPhi_k \times^*) \vI_k^C  + (\vI_k^C \vPhi_k  \,\crff) \right]\! \vPhi_j \nonumber \\
&= \vPhi_i\T\left[ (\vI_k^C \vPhi_k\, \crff) - \vB(\vPhi_k, \vI_k^C) \right]\vPhi_j\,.
\label{eq:chris3_simplified}
\end{align}
~\\[-3ex]

To turn these formulas into a recursive algorithm, let
\begin{align*}
\vBtmp_k = \vB(\vPhi_k, \vI_k^C)  ~~~{\textrm{and}}~~~
\vD_k = ( \vI_k^C \vPhi_k \crff) - \vBtmp_k\,.
\end{align*}
These quantities can be computed for all $k$ with $O(N)$ total cost. Considering \eqref{eq:chris1}, \eqref{eq:chris2}, and \eqref{eq:chris3_simplified}, we likewise define
\begin{align}
\vF_{1,jk} = \vBtmp_k \vPhi_j, ~~
\vF_{2,jk} = \vBtmp_k\T \vPhi_j,~{\textrm{and}}~~ 
\vF_{3,jk} = \vD_k \vPhi_j, \nonumber
\end{align}
which can be computed for all $j\preceq k$ in $O(Nd)$ total time. Finally, these intermediate quantities can be used to form
\begin{align*}
\Gamma_{ijk} &= \Gamma_{ikj} = \vPhi_i\T \vF_{1,jk}\,, \\
\Gamma_{jik} &= \Gamma_{jki} = \vPhi_i\T \vF_{2,jk}\,, \textrm{~and} \\
\Gamma_{kij} &= \Gamma_{kji} = \vPhi_i\T \vF_{3,jk}\,, 
\end{align*}
which can be computed for all ${i\preceq j \preceq k}$ with $O(Nd^2)$ total cost. Algo.~\ref{alg:chris} carries out these calculations in body-fixed coordinates, including all necessary transformation matrices.

\begin{algorithm}[t]
\caption{Christoffel Symbols Algorithm}
\begin{algorithmic}[1]
\setlength{\itemindent}{.01cm}
\addtolength{\algorithmicindent}{.01cm}
  
\setstretch{1.15}
\REQUIRE $\vq$
\FOR{$i=1$ to $N$}
\STATE $\vI_i^C = \vI_i$ 
\ENDFOR
\FOR{$k=N$ to $1$}
\STATE $\vBtmp=  \frac{1}{2}[ (\vPhi_k \times^*) \vI_k^C + (\vI_k^C \vPhi_k)\crff-   \vI_k^C (\vPhi_k \times )] $
\STATE $\vD =  (\vI_k^C\vPhi_k \crff) - \vBtmp $
\STATE $j=k$
\WHILE{$j>0$}
\STATE $\BF_1 = \vBtmp \, \vPhi_j$;~ $\BF_2 = \vBtmp\T \, \vPhi_j$;~$\BF_3 = \vD \, \vPhi_j$
\STATE $i = j$
\WHILE{$i>0$}
\STATE $\Gamma_{ijk} = \Gamma_{ikj} = \vPhi_i\T \BF_1$  
\STATE $\Gamma_{jik} = \Gamma_{jki} = \vPhi_i\T \BF_2$  
\STATE $\Gamma_{kij} = \Gamma_{kji} = \vPhi_i\T \BF_3$  
\STATE $\BF_1\!=\!\XM{i}{p(i)}\T \BF_1$;\, $\BF_2\!=\!\XM{i}{p(i)}\T \BF_2$;\, $\BF_3\!=\!\XM{i}{p(i)}\T \BF_3$\!\!
\STATE $i = p(i)$ 
\ENDWHILE
\STATE $\vBtmp = \XM{j}{p(j)}\T \, \vBtmp \, \XM{j}{p(j)}$ \\[.5ex]
\STATE $\vD = \XM{j}{p(j)}\T \, \vD \, \XM{j}{p(j)}$ 
\STATE $j = p(j)$ 
\ENDWHILE
\STATE $\vI_{p(k)}^C = \vI_{p(k)}^C + \XM{k}{p(k)}\T \, \vI_k^C \, \XM{k}{p(k)} $ 
\ENDFOR
\RETURN $\boldsymbol{\Gamma}$
\end{algorithmic}
\label{alg:chris}
\end{algorithm}

In Algo.~\ref{alg:chris}, the forward sweep (lines 1-3) initializes the composite inertia for each body. 
The nested while loops (lines 8-21) compute all the entries of $\Gamma_{ijk}$. There are two while loops to cover all permutations of indices $i \preceq j$ for a given $k$. Lines 18 and 19 propagate $\vBtmp$ and $\vD$ down the tree, and line 22 updates the composite inertias. 

\section{Results}

The algorithms were first prototyped in MATLAB and then implemented in C/C++ as an extension to RBDL \cite{Felis16}. The MATLAB implementations are available open source:
\[
\footnotesize
\text{\href{ https://github.com/ROAM-Lab-ND/spatial_v2_extended}{\tt www.github.com/ROAM-Lab-ND/spatial\_v2\_extended}}
\]
The Coriolis algorithm was first compared against RNEA for verification. From the equations of motion \eqref{eq:eom}, when gravity is 0 and $\vqdd=\mathbf{0}$, the torque vector returned by RNEA should equal the product of $\C(\q,\qd)$ and $\qd$. Randomized inputs for $\q$ and $\vqd$ were used to call RNEA and Algo.~\ref{alg:cor} under these conditions, with each entry of $\q$ randomized over $[0,2\pi]$ rad and each entry of $\qd$ randomized over $[0,10]$ rad/s. The RNEA torque vector was then compared to the product $\vC(\q,\vqd) \vqd$, with $\vC$ from Algo.~\ref{alg:cor}.  In the case of a 10 DoF serial kinematic chain,  a maximum error of $1.3\times10^{-11}$\,N$\cdot$m was observed over 100 trials. This check was repeated with 20 and 30 bodies, providing similar results (accurate up to $1.4 \times 10^{-9}$\,N$\cdot$m in the worst case) and ensuring the validity of the Coriolis matrix returned. Next, it was verified that Algo.~\ref{alg:cor} provides an admissible $\vC$ by checking that $\Hd=\C+\C \T$ for random inputs $\vq$ and $\vqd$ as before. A maximum error of $1.8\times10^{-12}$ N$\cdot$m$\cdot$s  was observed for the entries of $\Hd-\C-\C \T$ over the same set of tests.

The Christoffel symbols algorithm was then verified against the Coriolis matrix algorithm. The body-level factorization \eqref{eq:GunterB} was used to compute $\vCc(\vq,\vqd)$ using Algo.~\ref{alg:cor}. When Algo.~\ref{alg:chris} is called with the same $\q$, the returned symbols $\Gamma_{ijk}(\q)$  should satisfy $C_{ij} = \sum_k \Gamma_{ijk} \dot{q}_k$. Across the same random trails as before, the maximum residuals for $C_{ij} - \sum_k \Gamma_{ijk} \dot{q}_k$ were $1.6 \times 10^{-11}$ N$\cdot$m$\cdot$s, further supporting the correctness of the algorithms.

The computational cost of the algorithms was then benchmarked. Figs.~\ref{fig:serial}, \ref{fig:branch}, and \ref{fig:bipquad} show the required computation time for serial kinematic chains, binary trees, bipeds, and quadrupeds with varying numbers of bodies. Each data point represents an average over 100 random trials on an Intel i7 CPU (3.2GHz) running Ubuntu 18.04. For serial kinematic chains, the computational cost to compute the Christoffel symbols scaled approximately as $O(N^3)$ while the cost for the Coriolis matrix scaled as $O(N^2)$ (Fig.~\ref{fig:serial}). For binary trees (branched mechanisms with a branching factor of 2), the depth goes as $O(\log N)$, and so the time to compute the Coriolis matrix should increase as $O(Nd) = O(N \log N )$ while that of the Christoffel symbols should scale as $O(N d^2) = O(N (\log N)^2 )$. A polynomial fit to the log-log plots in Fig.~\ref{fig:branch} verifies that the order of the algorithms in this case is above $O(N)$ but below $O(N^2)$, in agreement with the theoretical complexity analysis. 
To further see the effects of branching, consider a quadruped topology with a main body and four serial-chain legs vs.~a biped topology with a main body and two serial-chain legs. For the same number of actuated joints, the algorithms are faster for a quadruped than a biped (Fig.~\ref{fig:bipquad}).

\begin{figure}[tbp]
\centerline{\includegraphics[width=.83 \columnwidth]{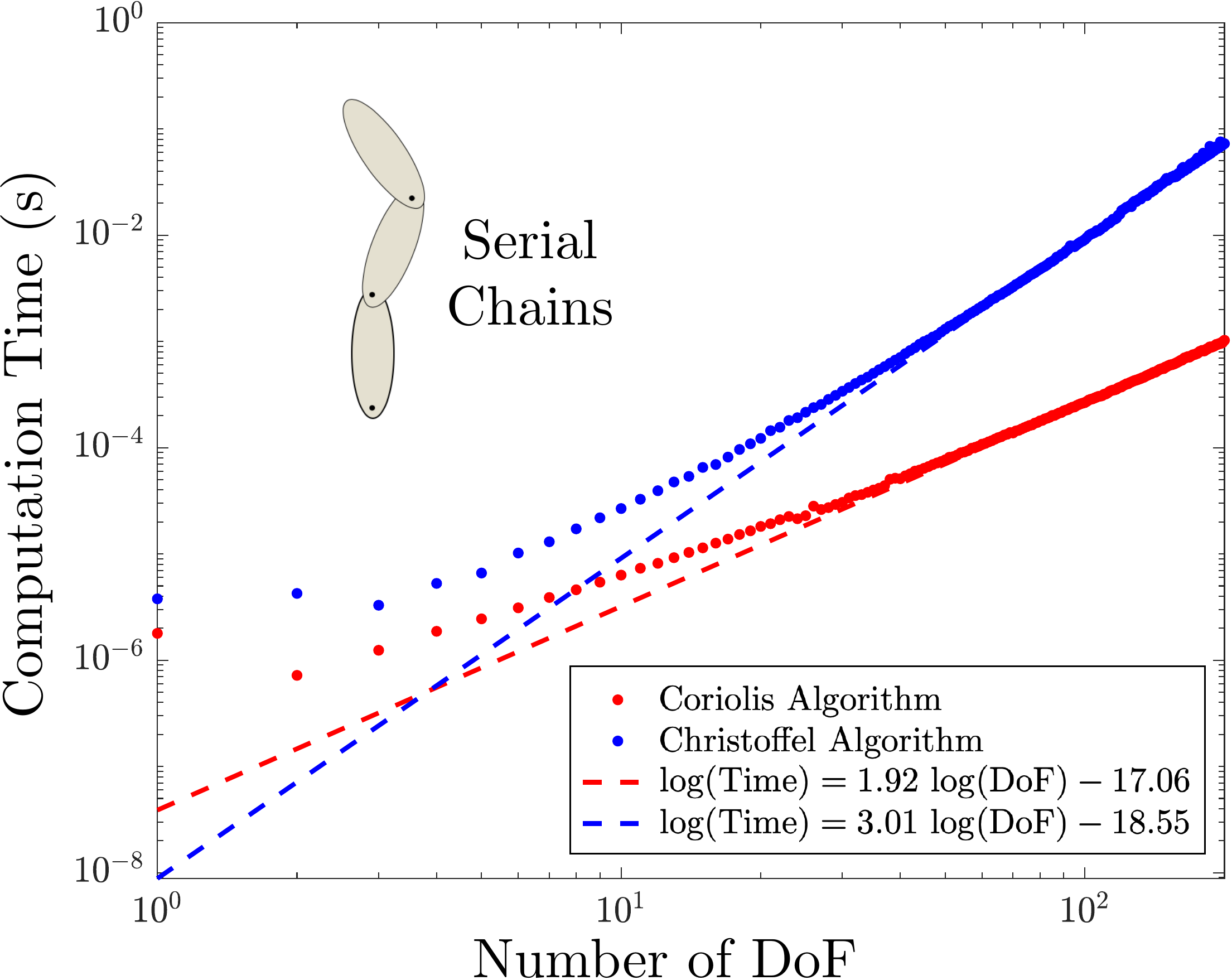}}
\caption{Computation times of Algo.~\ref{alg:cor} and \ref{alg:chris} versus number of DoF for serial chains connected by revolute joints.}
\label{fig:serial}
\end{figure}

\begin{figure}[tbp]
\centerline{\includegraphics[width=.83\columnwidth]{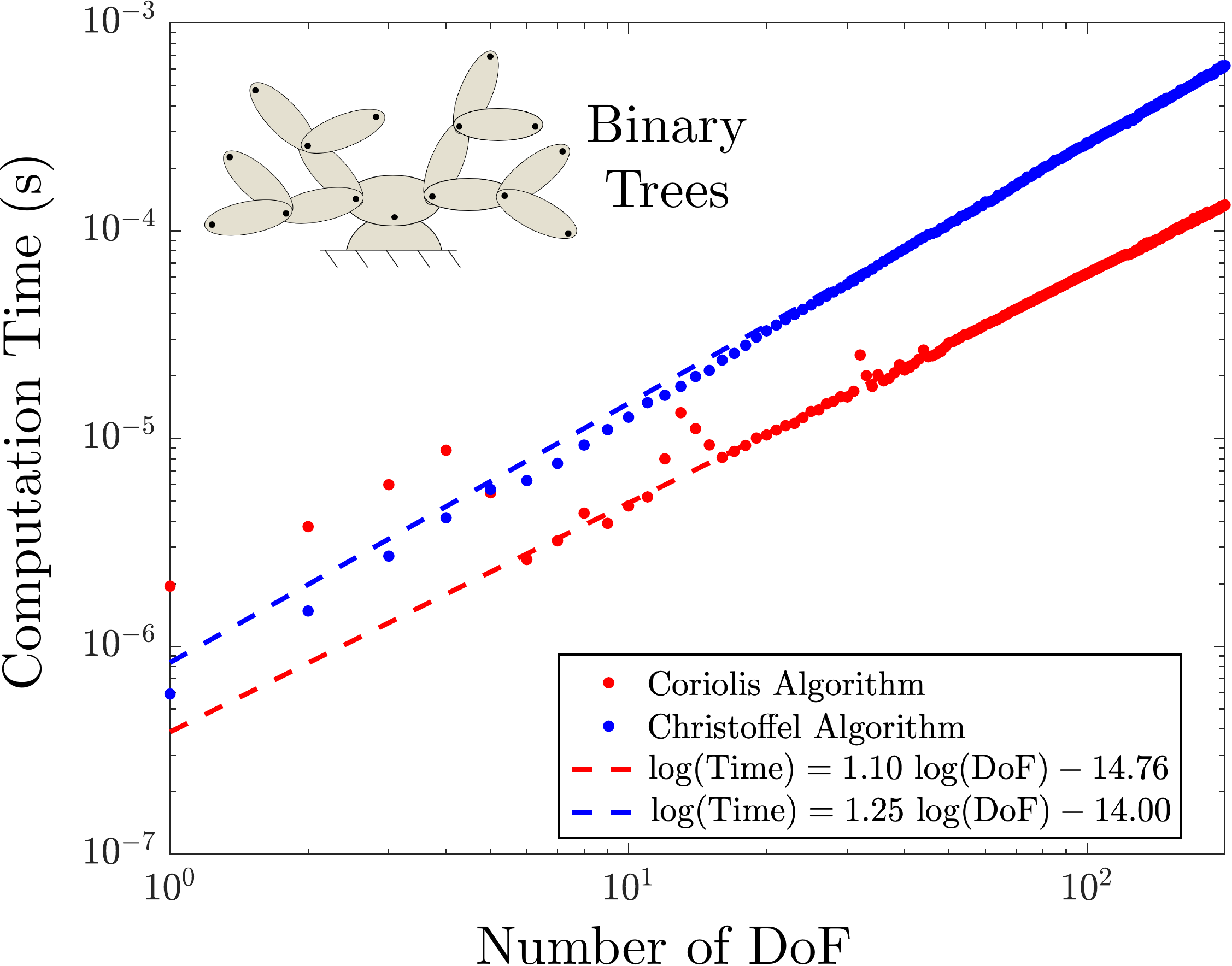}}
\caption{Computation times of Algo.~\ref{alg:cor} and \ref{alg:chris} versus number of DoF for binary trees connected by revolute joints.}
\label{fig:branch}
\end{figure}

\begin{figure}[!ht]
\centerline{\includegraphics[width = \columnwidth]{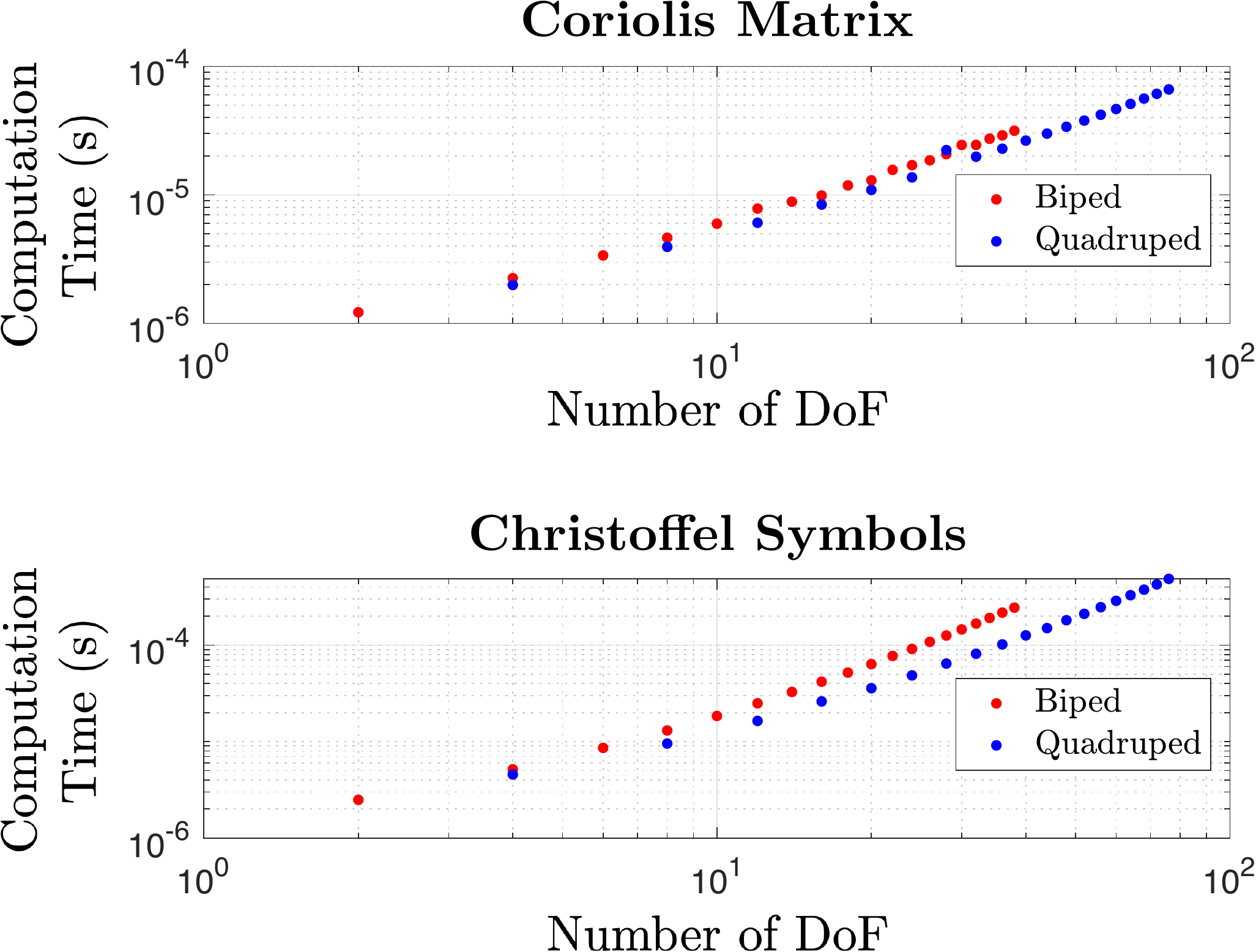}}
\caption{Computation times for Algorithms \ref{alg:cor} and \ref{alg:chris} versus number of joint DoF for biped and quadruped topologies.}
\label{fig:bipquad}
\end{figure}

In terms of run time, Algo.~\ref{alg:cor} and \ref{alg:chris} proved to be fast enough for online control loops that run at kHz rates. For rigid-body chains of 20 DoF, the algorithms take approximately 20 $\mu s$ to numerically evaluate the Coriolis matrix and 122 $\mu s$ to evaluate the Christoffel symbols (Table~\ref{tbl:res}). Again, shorter computation times for the Christoffel symbols occur with trees that have branching. For example, for a 20-DoF binary tree, all $\Gamma_{ijk}$ can be evaluated in 33 $\mu s$. 

\begin{table}[tb]
\caption{Computation time required for Algo.~\ref{alg:cor} and ~\ref{alg:chris} for selected rigid-body systems.}
\begin{center}
\vspace{-3ex}
\small
\begin{tabular}{|c|c|c|}
\cline{1-3} 
 & $\vC~ (\mu s)$ & $\Gamma_{ijk}~(\mu s)$ \\
\hline
Serial Chain (20 DoF) & 18 & 122 \\
\hline
Binary Tree (20 DoF)& 10 & 33 \\
\hline
Biped (20 actuated DoF) & 13 & 64 \\
\hline
Quadruped (20 actuated DoF)& 10 & 37 \\
\hline
\end{tabular}
\label{tbl:res}
\end{center}
\vspace{-5px}
\end{table}

The benefits of the algorithm can also be observed compared to when computing $\Gamma_{ijk}$ symbolically. A common strategy is to first obtain the mass matrix symbolically (e.g., using the Composite-Rigid-Body Algorithm with symbolic inputs \cite{Walker82,Featherstone08}) and then take its partial derivatives to form the Christoffel symbols via \eqref{eq:GammaEq}. These symbolic calculations become very complex for systems with more than just a few bodies. This approach was carried out using symbolic variables in MATLAB and was timed for binary trees with $N =$ 5, 10, and 15 bodies. This approach was compared to calling Algo.~\ref{alg:chris} with symbolic inputs (Table \ref{tbl:sym}). The two strategies were found roughly comparable but significantly slower than when calling Algo.~\ref{alg:chris} with numeric inputs. In this case, if a symbolic result is desired, our algorithm is preferred over the conventional formula \eqref{eq:GammaEq}, and these benefits increase with additional bodies. For serial chains, it was surprisingly found that evaluating the partials of $\vH$ symbolically was preferable to running Algo.~\ref{alg:chris} with symbolic inputs. However, beyond $N = 15$, both symbolic approaches take hours to run, and quickly thereafter do not complete within a day. This performance makes the numerical evaluation of the Christoffel symbols more remarkable. In the time it takes to run Algo.~\ref{alg:chris} symbolically for a 10-DoF planar serial chain ($26.8$ minutes), Algo.~\ref{alg:chris} can numerically evaluate all symbols $\Gamma_{ijk}$ more than ten million times.

\begin{table}[tb]
\caption{Computation times required to symbolically compute the Christoffel symbols for binary trees.}
\begin{center}
\vspace{-3ex}
\begin{tabular}{|c|c|c|c|}
\cline{1-4} 
 $N$ & $5$ & $10$ & $15$ 
 \\
\hline
Eq.~\eqref{eq:GammaEq}, Symbolic Partials of $\H$  & 1.23s & 7.83s & 27.1s 
\\
\hline
Algo.~\ref{alg:chris} with Symbolic Inputs & 0.97s & 3.11s & 5.88s 
\\\hline
\end{tabular}
\label{tbl:sym}
\end{center}
\vspace{-10px}
\end{table}

\section{Conclusions}

This paper developed efficient algorithms to numerically calculate the Coriolis matrix and its associated Christoffel symbols. Expressions for each entry in the Coriolis matrix were given in terms of composite quantities of the spatial inertia $\vI_i$ and  body-level Coriolis terms $\vB( \bv_i, \vI_i)$, and a recursive algorithm was derived based on these results. By considering partial derivatives of the expressions for $\C_{ij}$, closed-form expressions for the Christoffel symbols were presented, and a recursive algorithm to compute all the symbols was derived. The Coriolis matrix and Christoffel symbols algorithms have complexity $O(Nd)$ and $O(Nd^2)$, respectively, and proved to be fast. For a 20 DoF biped it took 13 $\mu$s to compute $\C$ and 64 $\mu$s for $\Gamma_{ijk}$. Due to the effects of branching, the algorithms are faster for a 20 DoF quadruped, where it took only 10 $\mu$s for $\C$ and 37 $\mu$s for $\Gamma_{ijk}$. Given the scalability and speed of these algorithms, we conclude that they are viable for implementation in real-time control loops and other dynamics applications.

\begin{acknowledgment}
The authors acknowledge NSF Grant CMMI 1835186 and ONR Award N0001420WX01278 (through a sub-award to Notre Dame) for partial support of this work. The authors thank Gianluca Garafalo and Christian Ott for stimulating discussions underpinning multiple remarks. The authors also gratefully acknowledge Jared Di Carlo for early work and discussions on the Christoffel symbols algorithm.
\end{acknowledgment}

%
% ...

\bibliographystyle{asmems4}

{\small

}
\pagebreak 

\appendix       
\section*{Appendix A: Spatial Vector Algebra in Coordinates}
\vspace{.5ex}
\noindent Consider body $i$ with spatial velocity $\v_i$ and frame $i$ attached to the body. The spatial velocity is expressed in the basis associated with frame $i$ as
\begin{align}
{}^i\{\bv_i\} &= \BM
    {}^i\{ \vom_i \} \\
    {}^i \{ \vv_i \}
\EM\,,\nonumber
\end{align}
where $\vom_i\in \mathbb{R}^3$ and $\vv_i\in \mathbb{R}^3$ are the angular velocity of frame $i$ and linear velocity of its origin \cite{Featherstone08}. Similarly, a spatial force vector $\Bf_i$ is expressed as
\begin{align}
^i\{ \Bf_i\} &= \BM
    ^i\{\vn_i\} \\
    ^i\{\vf_i\}
\EM \,, \nonumber
\end{align}
where $\vn_i$ is the moment about the origin of frame $i$ and $\vf_i$ is a linear force. 
The expression of motion vectors can be changed from frame $i$ to frame $j$ via multiplication by
\begin{align}
 \XM{j}{i} &= \BM
    \XR{j}{i} & \mathbf{0}   \\
    -\XR{j}{i}\bS(^{i}\{\bp_{j/i} \})  & \XR{j}{i}  
\EM, \nonumber
\end{align}
where $\XR{j}{i}\in \mathbb{R}^{3\times3}$ is the rotation matrix from frame $j$ to frame $i$, $^{i}\{\bp_{j/i}\} \in \mathbb{R}^3$ is the vector from the origin of frame $i$ to the origin of frame $j$, and $\bS(\{\bp\})$ is the skew-symmetric 3D cross-product matrix for $\{\bp\}=[p_x,p_y,p_z]^T \in \R^3$ as
\begin{align}
\bS(\{\bp\}) &= \left(\begin{smallmatrix}
  0 & -p_z & p_y  \\
    p_z & 0 & -p_x  \\
    -p_y & p_x & 0 
\end{smallmatrix}\right)
. \nonumber
\end{align}

\noindent In coordinates, the spatial cross-product matrix
\begin{align}
\{ (\bv\times) \} &= \BM
    \bS(\{\vom\}) & \boldmath{0}   \\
    \bS(\{\vv\})  & \bS(\{\vom\})  
\EM. \nonumber
\end{align}
In a similar fashion we can express the $\crff$ operator as
\begin{align}
\{ (\Bf\crff )\} &= \BM
    -\bS(\{\vn\}) & -\bS(\{\vf\})   \\
    -\bS(\{\vf\})  & \boldmath{0}  
\EM. \nonumber
\end{align}
The spatial inertia of body $i$ is expressed in coordinates as
\begin{align}
{}^i\{ \vI_i \} &= \BM
    {}^i\{\,\overline{\!\bI}_i\} & m_i\bS({}^i\{\bc_i\})   \\
    m_i\bS({}^i\{\bc_i\})\T  & m_i\bone_3 
\EM \nonumber
\end{align}
where $m_i$ is the mass of body $i$, ${}^i\{\bc_i\} \in \R^3$ is the vector to its CoM
, ${}^i\{\,\overline{\!\bI}_i\} \in \R^{3\times3}$ the Cartesian inertia tensor about the coordinate origin, and $\bone_3\in \mathbb{R}^{3\times3}$ is the identity matrix. 

\vspace{-1ex}
\section*{Appendix B: Supplemental Derivation of \eqref{eq:ForExtraDerivation}}
\vspace{.5ex}

\noindent Expanding $\va_k$ and $\v_k$ from \eqref{eq:RNEA}
\begingroup\makeatletter\def\f@size{9}\check@mathfonts
\begin{align}
\btau_i &= \sum_{\sumafter{k}{i}} \vPhi_i\T \vI_k \left( \sum_{j \preceq k} \vPhi_j \vqdd_j + \vPhid_j \vqd_j \right) + \vPhi_i\T \vB_k \left(\sum_{\sumbefore{j}{k}} \vPhi_j \vqd_j \right) \nonumber \\
&=\sum_{\sumafter{k}{i}} \sum_{\sumbefore{j}{k}} \vPhi_i\T \vI_k \vPhi_j \vqdd_j  \nonumber \\ &~~~~~~+ \left( \vPhi_i\T \vI_k \vPhid_j + \vPhi_i\T \vB_k \vPhi_j\right) \vqd_j \nonumber\,.
\end{align}
\endgroup
This expression sums $(k,j)$ over the set
\begin{align*}
S(i) &= \{ (k,j)~|~\sumafter{k}{i} \textrm{ and } \sumbefore{j}{k}\}\,\\
&= \{ (k,j) ~|~\sumrel{j}{i} \textrm{ and } \sumafter{k}{\ceil{ij}} \}\,.
\end{align*}
Applying this relationship, it follows that
\begin{align}
\btau_i &= \sum_{j \sim i} \sum_{\sumafter{k}{\ceil{ij}}} \vPhi_i\T \vI_k \vPhi_j \vqdd_j + \left( \vPhi_i\T \vI_k \vPhid_j + \vPhi_i\T \vB_k \vPhi_j\right) \vqd_j \nonumber \\
& = \sum_{j \sim i} \vPhi_i\T \vI_{\ceil{ij}}^C \vPhi_j \vqdd_j + \left( \vPhi_i \T \vI_{\ceil{ij}}^C \vPhid_j + \vPhi_i\T \vB_{\ceil{ij}}^C \vPhi_j \right) \vqd_j \nonumber\,.
\end{align}



\end{document}